\documentclass{article}

\usepackage[preprint]{neurips_2022}

\usepackage[utf8]{inputenc} % allow utf-8 input
\usepackage[T1]{fontenc}    % use 8-bit T1 fonts
\usepackage{hyperref}       % hyperlinks
\usepackage{url}            % simple URL typesetting
\usepackage{booktabs}       % professional-quality tables
\usepackage{amsfonts}       % blackboard math symbols
\usepackage{nicefrac}       % compact symbols for 1/2, etc.
\usepackage{microtype}      % microtypography
\usepackage{xcolor}         % colors
\usepackage{amsmath, amsfonts, amsthm}
\usepackage{breqn}
\DeclareMathOperator*{\argmin}{arg\,min}

\usepackage{graphicx}
\usepackage{wrapfig}
\usepackage{natbib}

\newtheorem{lemma}{Lemma}
\newtheorem{theorem}{Theorem}
\newtheorem{definition}{Definition}
\newtheorem{proposition}{Proposition}
\newtheorem{claim}{Claim}
\usepackage{mathtools}
\newcommand{\indep}{\perp \!\!\! \perp}

\title{Covariate-Balancing-Aware Interpretable Deep Learning Models for Treatment Effect Estimation}

\author{%
  Kan Chen \\
  Graduate Group of Applied Math and Computational Science\\
  University of Pennsylvania\\
  Philadelphia, PA 19104 \\
  \texttt{kanchen@sas.upenn.edu} \\
  \And
   Qishuo Yin \\
  Graduate Group of Applied Math and Computational Science\\
  University of Pennsylvania\\
  Philadelphia, PA 19104 \\
  \texttt{qsyin@sas.upenn.edu} \\
  \AND
  Qi Long \\
  Department of Biostatistics, Epidemiology and Informatics \\
  University of Pennsylvania \\
  Philadelphia, PA, 19104\\
  \texttt{qlong@upenn.edu} \\
}

\begin{document}

\maketitle

\begin{abstract}
Estimating treatment effects is of great importance for many biomedical applications with observational data. Particularly, interpretability of the treatment effects is preferable for many biomedical researchers.  In this paper, we first provide a theoretical analysis and derive an upper bound for the bias of average treatment effect (ATE) estimation under the strong ignorability assumption.  Derived by leveraging appealing properties of the Weighted Energy Distance, our upper bound is tighter than what has been reported in the literature. Motivated by the theoretical analysis, we propose a novel objective function for estimating the ATE that uses the energy distance balancing score and hence does not require correct specification of the propensity score model. We also leverage recently developed neural additive models to improve interpretability of deep learning models used for potential outcome prediction. We further enhance our proposed model with an energy distance balancing score weighted regularization. The superiority of our proposed model over current state-of-the-art methods is demonstrated in semi-synthetic experiments using two benchmark datasets, namely, IHDP and ACIC.
\end{abstract}

\section{Introduction} \label{sec: intro}

We consider the problem of estimating treatment effects using observational data in this paper. In causal inference, observational data refers to the information obtained from a sample or population that independent variables cannot be controlled over by the researchers because of ethical concerns or logistical constraints \citep{rosenbaum2010design}. In contrast to observational data, experimental data from randomized controlled trials are viewed as a ``gold standard'' in causal inference but they are very expensive and in some cases infeasible to conduct.

However, estimation of causal effect from observational data is complicated by potential confounding that may influence treatment receipt and/or outcome of interest. This work is developed under the ``strong ignorability'' assumption, that is, there is no unobserved confounding. We consider the estimation of the effect of treatment $A$ (e.g. assignment of a certain therapy) on an outcome (e.g. recover or not) while adjusting for observed covariates $X$ (e.g. demographic status of patients).

One popular approach for estimation of treatment effects consists of two steps: first, fitting models for expected outcomes and propensity score, respectively; second, plugging the fitted models into a downstream estimator of the treatment effects.  Neural networks is a powerful tool for the first step because of its impressive predictive performance \citep{shalit2017estimating,johansson2016learning, louizos2017causal,alaa2017deep,alaa2017bayesian,schwab2018perfect,yoon2018ganite,farrell2018deep,shi2019adapting,bica2020estimating, kaddour2021causal}. Examples such as Dragonnet, a deep neural network architecture by \citet{shi2019adapting}, and SITE, a deep representation learning for ITE estimation by \citet{yao2018representation} have demonstrated the outstanding performance of deep neural network in the estimation of treatment effect,  However, one limitation of such models is that they are a ``black-box" and difficult to interpret or explain \citep{agarwal2020neural}.

The main contributions of this paper can be summarized as follows. We adopt the approach of viewing the error in predicting outcomes under the observed treatment assignment actually received as the training error and the error in predicting outcomes under the counter-factual treatment assignment (i.e. the opposite of the observed treatment assignment) received as the testing error. We are the first to leverage the appealing properties of the energy distance for measuring distance between distributions to derive a bound on the error in estimating average treatment effect (ATE) that is tighter than existing works \citep{shalit2017estimating}. Motivated by our theoretical results, we propose a novel objective function for estimating the ATE that uses the energy distance balancing score and hence does not require correct specification of the propensity score model. We also leverage recently developed neural additive models (NAMs) to improve interpretability of deep learning models used for potential outcome prediction. We further enhance our model with an energy distance balancing score weighted regularization. We demonstrate the superiority of the proposed model over current state-of-the-art methods in numerical experiments using two well-known benchmark datasets, namely, IHDP \citep{hill2011bayesian} and ACIC \citep{mathews1998infant}.

 \section{Preliminaries} \label{sec: prlim}

%\subsection{Setup} 
To fix ideas, we consider the estimation of ATE of a binary treatment. Consider a random sample $\{(Y_i, A_i, \mathbf{X}_i) \}$ of size $n$ from a target population, where $Y_i$ is the observed outcome of $i$-th unit, $A_i \in \{0,1\}$ is a binary indicator of receiving treatment, and $\mathbf{X}_i = (x_{1i}, x_{2i}, \cdots,x_{pi}) \in \mathbb{R}^p$ is a $p$-dimensional covariate vector of unit $i$. We also let $n_1$ be the number of treated units, and $n_0$ be the number of control units with $n_0 = n - n_1$. Following the potential outcome framework by \citet{rubin1974estimating,splawa1990application}, each unit $i$ has two potential outcomes, $Y_i(1)$ and $Y_i (0)$ that unit $i$ would have under the treatment and control, respectively. It follows that the observed outcome $Y_i = A_i Y_i (1) + (1 - A_i) Y_i(0)$.  We make the standard stable unit treatment value assumption (SUTVA), that is, ``the observation on one unit should be unaffected by the particular assignment of treatment to other units'' (\citet{cox1958planning}). Under SUTVA, the observed outcome is consistent with the potential outcome in the sense that $Y_i = Y_i (A_i)$. We further assume that the treatment assignment mechanism is \emph{strongly ignorable}, i.e. $\{Y_i (0), Y_i(1) \} \indep A_i | \mathbf{X}_i$, and make the \emph{positivity} assumption, i.e. $0 < \mathbb{P}(A_i = 1 | \mathbf{X}_i = \mathbf{x}) < 1$. In other words, this is saying that there is no unmeasured confounders and every unit has a chance to receive the treatment. 

Using the notation of the potential outcome framework, the ATE is defined as $\tau \equiv \mathbb{E} (Y(1) - Y(0)) $ which could be estimated as $\hat \tau = \frac{1}{n}\sum_{i=1}^n \left( Y_i (1) - Y_{i}(0) \right)$ if both potential outcomes were observed. Additionally, we let $F_1 (\mathbf{x}) \equiv \mathbb{P} (\mathbf{X} \leq \mathbf{x} | A = 1 ),  F_0 (\mathbf{x}) \equiv \mathbb{P} (\mathbf{X} \leq \mathbf{x} | A = 0 )$ denote the cumulative distribution function (CDF) of covariate $\mathbf{x}$ in the treatment group and the control group, respectively. It follows that the CDF of $\mathbf{X}$ in the entire population is $F(\mathbf{x}) \equiv \mathbb{P} (\mathbf{X} \leq \mathbf{x} ) = F_1 (\mathbf{x}) P_1 + F_0 (\mathbf{x}) P_0$, where $P_1 \equiv \mathbb{P}(A = 1) and P_0 \equiv \mathbb{P}(A = 0)$. If we further let $\mu_1 (\mathbf{X}_i) \equiv \mathbb{E} (Y(1) | \mathbf{X}_i), \mu_0 (\mathbf{X}_i) \equiv \mathbb{E} (Y(0) | \mathbf{X}_i)$ be the expected conditional treated and control outcomes based on the observed covariate $\mathbf{X}_i$ for unit $i$, then the ATE can also be written as
\begin{align*}
    \tau = \int_{\mathbb{R}^p} \left( \mu_1 (\mathbf{x}) - \mu_0 (\mathbf{x})     \right) dF(\mathbf{x}).
\end{align*}
And the individual treatment effect (ITE) for unit $i$ is
\begin{align*}
    \phi(\mathbf{X}_i) = \mu_1 (\mathbf{X}_i) - \mu_0 (\mathbf{X}_i).
\end{align*}
Hence, the estimated ITE and ATE can be written as
\begin{align*}
        &\hat{\phi} (\mathbf{X}_i) = \hat{\mu}_1 (\mathbf{X}_i) - \hat{\mu}_0 (\mathbf{X}_i) \\
        &\hat{\tau} = \int_{\mathbb{R}^p} \hat{\phi}(\mathbf{x}) dF_n(\mathbf{x})
\end{align*}
where $\hat{\mu}_1 (\mathbf{X}_i)$ and $\hat{\mu}_0 (\mathbf{X}_i)$ are the predicted expected outcomes in the treated and control groups given observed covariate $\mathbf{X}_i$ and $F_n (\mathbf{x}) = \frac{1}{n}\sum_{i=1}^n \mathbf{1}(\mathbf{X}\leq \mathbf{x})$ is the empirical CDF. Of note, the ATE measures the difference in average outcomes between units assigned to the treatment and units assigned to the control and captures population-level causal effects, whereas the ITE captures the individual level causal effect defined as $Y_i (1) - Y_i (0)$. While ITE is unique to an individual and may not be described exactly by a set of units,  the \emph{conditional average treatment effect} (CATE), defined as $\mu_1 (\mathbf{X}) - \mu_0 (\mathbf{X})$ can be used to describe the ATE within a subgroup of units defined by $\mathbf{X}$. %In the next subsection, we will state some definitions for conducting theoretical analysis for average treatment effect. 

\subsection{Definitions} \label{subsec: def}

To facilitate the theoretical analysis for the upper bound of individual treatment effect, we state the following definitions.

\begin{definition} \label{def: f and cf loss}
Let $L: \mathcal{Y} \times \mathcal{Y} \rightarrow \mathbb{R}^{+}$ be the squared loss function $L(y,y') = \lVert y - y' \rVert_2$. Then the expected pointwise loss is
\begin{align*}
   l_{\hat{\mu}_a} (\mathbf{x}) \equiv \int_{\mathcal{Y}} L(\mu_a (\mathbf{x}) , \hat{\mu}_a (\mathbf{x}) ) p(Y(a) | \mathbf{x} ) dY(a).
\end{align*}
The expected factual and counterfactural losses are:
\begin{align*}
    R_F (\hat{\mu}_a) &\equiv \int_{\mathbb{R}^p \times \{0,1\} } l_{\hat{\mu}_a} (\mathbf{x}) d F_a (\mathbf{x}) da \\
    R_{CF} (\hat{\mu}_a) &\equiv \int_{\mathbb{R}^p \times \{0,1\} } l_{\hat{\mu}_a} (\mathbf{x}) d F_{1-a} (\mathbf{x}) da. \\
    R_F (\hat{\mu}_a, \mathbf{w}) &\equiv \int_{\mathbb{R}^p \times \{0,1\} } l_{\hat{\mu}_a} (\mathbf{x}) d F_{n,a,\mathbf{w}} (\mathbf{x}) da \\
    R_{CF} (\hat{\mu}_a, \mathbf{w}) &\equiv \int_{\mathbb{R}^p \times \{0,1\} } l_{\hat{\mu}_a} (\mathbf{x}) d F_{n,1-a,\mathbf{w}} (\mathbf{x}) da
\end{align*}
where $\hat{\mu}_a (\mathbf{x})$ is the predicted expected conditional outcomes at treatment level $A = a$ given observed covariate $\mathbf{x}$, and $F_{n,a, \mathbf{w}} (\mathbf{x}) = \sum_{i=1}^n w_i \mathbf{1} (\mathbf{X}_i \leq \mathbf{x}, A_i = a) / n_a $ is the weighted ECDF of data points at treatment level $a$ with weights $\mathbf{w} = (w_1, \cdots, w_n)$ that satsify $\sum_{i=1}^n w_i A_i = n_1, \sum_{i=1}^n w_i (1 - A_i) = n_0$, and $w_i \geq 0$.
\end{definition}

To help understand Definition \ref{def: f and cf loss}, consider the following example. If $\mathbf{x}$ is the demographic status of patients, $a$ is the treatment level, and $Y(a)$ is the potential outcome such as recovery or not given treatment level $a$, then $R_F$ measures the predictive performance under the treatment received. On the other hand, $R_{CF}$ measures the predictive performance under a counter-factual treatment assignment that is the opposite of the treatment received.  From the perspective of machine learning, $R_F$ can be viewed as the training error and $R_{CF}$ can be viewed as the testing error.   

\begin{definition}
The expected loss in estimation of average treatment effect is
\begin{align*}
    R_{ATE} (\hat{\mu}_a) &\equiv \int_{\mathbb{R}^p} ( \hat{\phi} (\mathbf{x}) - \phi (\mathbf{x}) )^2 d F(\mathbf{x}). \\
    R_{ATE,n} (\hat{\mu}_a) &\equiv \int_{\mathbb{R}^p} ( \hat{\phi} (\mathbf{x}) - \phi (\mathbf{x}) )^2 d F_n(\mathbf{x}). 
\end{align*}
When $n \rightarrow \infty$, $R_{ATE} (\hat{\mu}_a)  =  R_{ATE,n} (\hat{\mu}_a) $.
\end{definition}

Our theoretical analysis relies heavily on the notion of the Weighted Energy Distance metric by \citet{huling2020energy}, which is a distance metric between two probability distributions. For two distributions $G (\mathbf{x}), H (\mathbf{x})$ defined on $\mathbb{R}^p$, we define the energy distance as follows.

\begin{definition}[Energy Distance \citep{cramer1928composition}]
The energy distance between two distributions $G$ and $H$ is
\begin{equation} \label{equ: energy int}
     \mathcal{E} (G, H) \equiv 2 \int_{\mathbb{R}^p} ( G(\mathbf{x}) - H(\mathbf{x})  )^2 d \mathbf{x}.
\end{equation}
It can also be written as
\begin{equation} \label{equ: energy sum}
     \mathcal{E} (G, H) = 2 \mathbb{E} \| \mathbf{Z} - \mathbf{V} \|_2  - \mathbb{E} \| \mathbf{Z} - \mathbf{Z}' \|_2 - \mathbb{E} \| \mathbf{V} - \mathbf{V}' \|_2
\end{equation}
where $\mathbf{Z},\mathbf{Z}' \overset{i.i.d}{\sim} G, \mathbf{V},\mathbf{V}' \overset{i.i.d}{\sim} H  $. See proofs in \citet{szekely2003statistics} for the equivalence of (\ref{equ: energy int}) and (\ref{equ: energy sum}). When both $G$ and $H$ are empirical cumulative distribution functions (ECDF), i.e. $G_n$ is the ECDF of $\{\mathbf{Z}_i \}_{i=1}^n$, $H_m$ is the ECDF of $\{\mathbf{V}_i \}_{i=1}^m$, then
\begin{align*}
     \mathcal{E} &(G_n, H_m) = \frac{2}{nm}\sum_{i=1}^n \sum_{j=1}^m \|\mathbf{Z}_i - \mathbf{V}_j  \|_2 -  \frac{1}{n^2}\sum_{i=1}^n \sum_{j=1}^m \|\mathbf{Z}_i - \mathbf{Z}_j  \|_2 -  \frac{2}{m^2}\sum_{i=1}^n \sum_{j=1}^m \|\mathbf{V}_i - \mathbf{V}_j  \|_2. 
\end{align*}
\end{definition}
Motivated by the definition of energy distance, \citet{huling2020energy} proposed a weighted modification of this distance metric as follows.
\begin{definition} \label{def: weighted energy distance}
The weighted energy distance between $F_{n,a, \mathbf{w}}$ and $F_n$ is defined as
\begin{align*}
         \mathcal{E} &(F_{n,a, \mathbf{w}}, F_n) = \frac{2}{n_a n}\sum_{i=1}^n \sum_{j=1}^n w_i \mathbf{1}(A_i = a) \|\mathbf{X}_i - \mathbf{X}_j  \|_2 \\ 
     &-  \frac{1}{n_a^2}\sum_{i=1}^n \sum_{j=1}^n w_i w_j \mathbf{1}(A_i = A_j = a) \|\mathbf{X}_i - \mathbf{X}_j  \|_2 -  \frac{2}{n^2}\sum_{i=1}^n \sum_{j=1}^n \|\mathbf{X}_i - \mathbf{X}_j  \|_2. 
\end{align*}
\end{definition}
In other words, $\mathcal{E} (F_{n,a, \mathbf{w}}, F_n)$ is the energy distance between the ECDF of a sample $\{\mathbf{X}_i \}_{i=1}^n$ and a weighted ECDF of $\{\mathbf{X}_i \}_{i: A_i = a}$. The weight $\mathbf{w}$ here can be viewed as a balancing score which can be used for substituting the propensity score, and is henceforth referred to as \emph{Energy Distance Balancing Score}. More discussions about this balancing score are included in Section \ref{sec: main}.

\section{Main Result} \label{sec: main}

\subsection{Bias in Estimation of Average Treatment Effect}
We first state the results bounding the counter-factual loss, which lays the foundation for bounding the bias for estimation of ATE. We consider both the unweighted version and the weighted version. The proofs and details can be found in the supplementary material.

\begin{lemma} \label{lemma: CF}
Following the notation defined in Section~\ref{sec: prlim}, we have
\begin{equation}\label{equ:R_CF}
\begin{split}
     R_{CF} (\hat{\mu}_a)  &\leq P_0 R_F (\hat{\mu}_1)  + P_1 R_{F} (\hat{\mu}_0) + C_{l_{\hat{\mu}_a} } \left(\sqrt{\mathcal{E}(F_1, F)} + \sqrt{\mathcal{E}(F_0, F)} \right). 
\end{split}
\end{equation}
where $C_{l_{\hat{\mu}_a} } = \|\partial^p [-P_0 l_{\hat{\mu}_1} + P_1 l_{\hat{\mu}_0} ] /\partial \mathbf{x}^p  )  \|_2/\sqrt{2}$ is a constant. 
Furthermore, if we equip Inequality (\ref{equ:R_CF}) with weights $\mathbf{w}$ such that $\sum_{i=1}^n w_i A_i = n_1, \sum_{i=1}^n w_i (1 - A_i) = n_0, w_i \geq 0$, when $n \rightarrow \infty$, we then have
\begin{equation} \label{equ:R_CF weighted}
\begin{split}
       R_{CF} (\hat{\mu}_a, \mathbf{w}) &\leq P_0 R_F (\hat{\mu}_1, \mathbf{w}) + P_1 R_F (\hat{\mu}_0, \mathbf{w}) + C_{l_{\hat{\mu}_a} } \left( \sqrt{\mathcal{E} (F_{n,1, \mathbf{w}}, F_{n}) } + \sqrt{\mathcal{E} (F_{n,0, \mathbf{w}}, F_{n})} \right).
\end{split}
\end{equation}
\end{lemma}

\begin{theorem} \label{thm: bound error ATE}
Under the conditions of Lemma \ref{lemma: CF}, we have
\begin{equation} \label{equ:ATE unweighted}
\begin{split}
     R_{ATE} (\hat{\mu}_a) &\leq 2 \Bigg( R_F (\hat{\mu}_1)  + R_{F} (\hat{\mu}_0) +  C_{l_{\hat{\mu}_a} } \left(\sqrt{\mathcal{E}(F_1, F)} + \sqrt{\mathcal{E}(F_0, F)} \right) \Bigg).
\end{split}
\end{equation}
Furthermore, if we equip Inequality (\ref{equ:ATE unweighted}) with weights $\mathbf{w}$ such that $\sum_{i=1}^n w_i A_i = n_1, \sum_{i=1}^n w_i (1 - A_i) = n_0, w_i \geq 0$, when $n \rightarrow \infty$, we then have
\begin{equation} \label{equ:ATE weighted}
    \begin{split}
     & \quad  R_{ATE} (\hat{\mu}_a) =  R_{ATE, n} (\hat{\mu}_a) \\
       & \leq 2 \Bigg( R_F (\hat{\mu}_1,\mathbf{w})  + R_{F} (\hat{\mu}_0,\mathbf{w}) +  C_{l_{\hat{\mu}_a} } \left(\sqrt{\mathcal{E}(F_{n,1,\mathbf{w}}, F_n)} + \sqrt{\mathcal{E}(F_{n,0,\mathbf{w}}, F_n)} \right)  \Bigg).
    \end{split}
\end{equation}
\end{theorem}

Similar in spirit with the analysis approach in \citet{shalit2017estimating, johansson2020generalization}, the basic idea of the proof of Theorem \ref{thm: bound error ATE} is to bound $R_{CF}(\cdot)$ using $R_{F}(\cdot)$ and distance between the treated and control distributions, the latter of which exploits the properties of energy distance investigated in \citet{huling2020energy}. To better understand Theorem~\ref{thm: bound error ATE}, we can decompose the bound into two parts, namely, training error for factual outcomes, and the distance between treated and control distributions. 

We highlight the differences between our results and the prior results. First, our target estimator is for the ATE in comparison to the ITE in \citet{shalit2017estimating} and the CATE in \citet{johansson2020generalization}. Second, we use weighted energy distance between $F_{n,a, \mathbf{w}}$ and $F_n$ in the analysis of our bound. The benefit of using weighted energy distance comes from the key observation that if $\mathbf{w}^*$ satisfies 
\begin{equation} \label{equ: w^*}
\begin{split}
            \mathbf{w}^* &\in \argmin_{\mathbf{w}}\left(\sqrt{\mathcal{E}(F_{n,1,\mathbf{w}}, F_n)} + \sqrt{\mathcal{E}(F_{n,0,\mathbf{w}}, F_n)} \right) \\
    &\sum_{i=1}^n w_i A_i = n_1, \sum_{i=1}^n w_i (1 - A_i) = n_0, w_i \geq 0,
\end{split}
\end{equation}

then $\lim_{n \rightarrow \infty} \mathcal{E}(F_{n,1,\mathbf{w}^*}, F_n) = 0$, and $\lim_{n \rightarrow \infty} \mathcal{E}(F_{n,0,\mathbf{w}^*}, F_n) = 0$
almost surely. This essentially adapts the proof of Theorem 3.1 in \citet{huling2020energy} to our loss function. It follows that Inequality~(\ref{equ:ATE weighted}) becomes
\begin{align*}
     & \quad  R_{ATE} (\hat{\mu}_a) =  R_{ATE, n} (\hat{\mu}_a)  \leq 2 \Bigg( R_F (\hat{\mu}_1,\mathbf{w^*})  + R_{F} (\hat{\mu}_0,\mathbf{w^*})  \Bigg).
\end{align*}
which yields a much tighter bound when the training error for factual outcomes is minimized. 

Motivated by Theorem \ref{thm: bound error ATE}, we propose a new objective function equipped with the interpretable deep learning model to estimate the ATE using observational data.

\subsection{Proposed Objective Function}\label{sec:objective}

Before proposing our objective function, we first state the definition of balancing score and a classic result from \citet{rosenbaum1983central}.
\begin{definition}[Balancing Score]
A balancing score $b(\mathbf{X})$ is a function of the observed covariate $\mathbf{X}$ such that the conditional distribution of $\mathbf{X}$ given $b(\mathbf{X})$ is the same for treated and control units (i.e. $A \indep \mathbf{X} | b(\mathbf{X})$). 
\end{definition}
The weights $\mathbf{w}$ defined in Definition \ref{def: weighted energy distance} is a type of balancing score since $A \indep \mathbf{X}|\mathbf{w}$. Based on the definition of balancing score, we can state the classic result about the sufficiency of balancing score. 

\begin{theorem}[\citet{rosenbaum1983central}]
If the ATE is identifiable from observational data by adjusting for  $\mathbf{X}$ (i.e. $\tau = \mathbb{E}\left( \mathbb{E}(Y(1)|\mathbf{X}) - \mathbb{E}(Y(0)|\mathbf{X})  \right) $), then we have
\begin{align*}
    \tau = \mathbb{E}\left( \mathbb{E}(Y(1)|\mathbf{X}, \mathbf{w}) - \mathbb{E}(Y(0)|\mathbf{X}, \mathbf{w})   \right).
\end{align*}
\end{theorem}
Propensity score $P(A = a| \mathbf{X} = \mathbf{x})$ is a popular balancing score used for estimation of ATE. However, propensity score is  obtained from a treatment assignment model and does not explicitly balance treated and control distributions under finite samples. Plus, if the propensity score model is incorrectly specified, it may lead to biased estimation of propensity score. As a result, the estimation of ATE is likely to be biased. To address this limitation, we propose to use the weights $\mathbf{w}$ from the weighted energy distance as our balancing score in the objective function. The sufficiency of the balancing score implies that to obtain an unbiased estimation of ATE, it suffices to adjust for the observed covariate $\mathbf{X}$ that is relevant to the estimation of balancing score. Therefore, we train the model by minimizing
\begin{align*}
                    \hat{\theta}, \hat{\mathbf{w}} &= \argmin_{\theta,\mathbf{w}} \hat{R}(\mathbf{X}, \mathbf{A};\theta, \mathbf{w})  \\
    s.t. \quad &\sum_{i=1}^n w_i A_i = n_1, \quad \sum_{i=1}^n w_i (1 - A_i) = n_0, w_i \geq 0 
\end{align*}
where
\begin{equation} \label{equ: obj}
    \begin{split}
        \hat{R}(\mathbf{X},\mathbf{A};\theta, &\mathbf{w}) = \frac{1}{n} \sum_{i=1}^n \left(Y_i - Q^{\text{NN}}(\mathbf{X}_i, A_i; \theta)  \right)^2 + \alpha \left( \sqrt{\mathcal{E} (F_{n,1, \mathbf{w}}, F_{n})}  + \sqrt{ \mathcal{E} (F_{n,0, \mathbf{w}}, F_{n})} \right).
    \end{split}
\end{equation}
Here, $\alpha \in \mathbb{R}^+$ is a hyperparameter, and  $Q^{\text{NN}}$ represent the neural network models for estimating $\hat{Q}$, and $\hat{Q}$ implicitly contains $\mathbf{w}$. With the fitted model $\hat{Q}$ in hand, we can then estimate the ATE using some downstream estimator such as
\begin{align*}
        \hat{\tau} = \frac{1}{n} \sum_{i=1}^n \left( \hat{Q}(\mathbf{X}_i, 1; \theta) - \hat{Q}(\mathbf{X}_i, 0; \theta) \right).
\end{align*}

Under some mild conditions, the algorithmic convergence of training error for factual outcomes predictions is guaranteed under the Neural Tangent Kernel (NTK) regime (\citet{du2018gradient}).

\begin{proposition} \label{prop:NTK}
Let $\mathbf{w}^*$ be as defined in (\ref{equ: w^*}). We consider a neural network of the form $Q^{\text{NN}}(\mathbf{X}_i, A; \theta) = \frac{1}{\sqrt{m}} \sum_{r = 1}^m a_r \sigma((\theta^r)^\top (\mathbf{X}_i,A) ) $ where $\theta^r \in \mathbb{R}^p$ is the weight vectors in the first hidden layer connecting to the r-th neuron, $a_r \in \mathbb{R}$ is the output weight, $\sigma (\cdot)$ is the ReLU activation function, and $m$ is the width of hidden layer. Assuming that $\forall i, \|(\mathbf{X}_i,A_i) \|_2 = 1, |Y_i (a)| < C$ for some constant $C$, $m = \Omega (n^6/\lambda_0^4 \delta^3)$, and we i.i.d initialize $\theta^r \sim \mathcal{N}(\mathbf{0},\mathbf{I}), a_r \sim \text{Unif}\{-1,1\}$ for $r \in [m]$, then with at least probability at $1 - \delta$ over the initialization, we have the linear convergence rate for $\hat{R} (\mathbf{X}, \mathbf{A}; \theta, \mathbf{w}^*) $ under gradient descent algorithm. Here
$\lambda_0 = \lambda_{\min} (\mathbf{H}^\infty)$, where
\begin{align*}
    (\mathbf{H}^\infty)_{ij} =\left(\frac{1}{2} - \frac{\text{arccos}((\mathbf{X}_i,A_i)^\top (\mathbf{X}_i,A_i) ) }{2 \pi}   \right) 
    \left((\mathbf{X}_i,A_i)^\top (\mathbf{X}_i,A_i)  \right).
\end{align*}
\end{proposition}

Detailed explanation of the NTK theory can be found in \citet{du2018gradient, bu2021dynamical}. The algorithmic convergence of training error provides much tighter bound for the error in estimation of ATE, lending support to the superior performance of the proposed method. 

\begin{wrapfigure}{l}{0.5\textwidth}
  \begin{center}
    \includegraphics[width=0.48\textwidth]{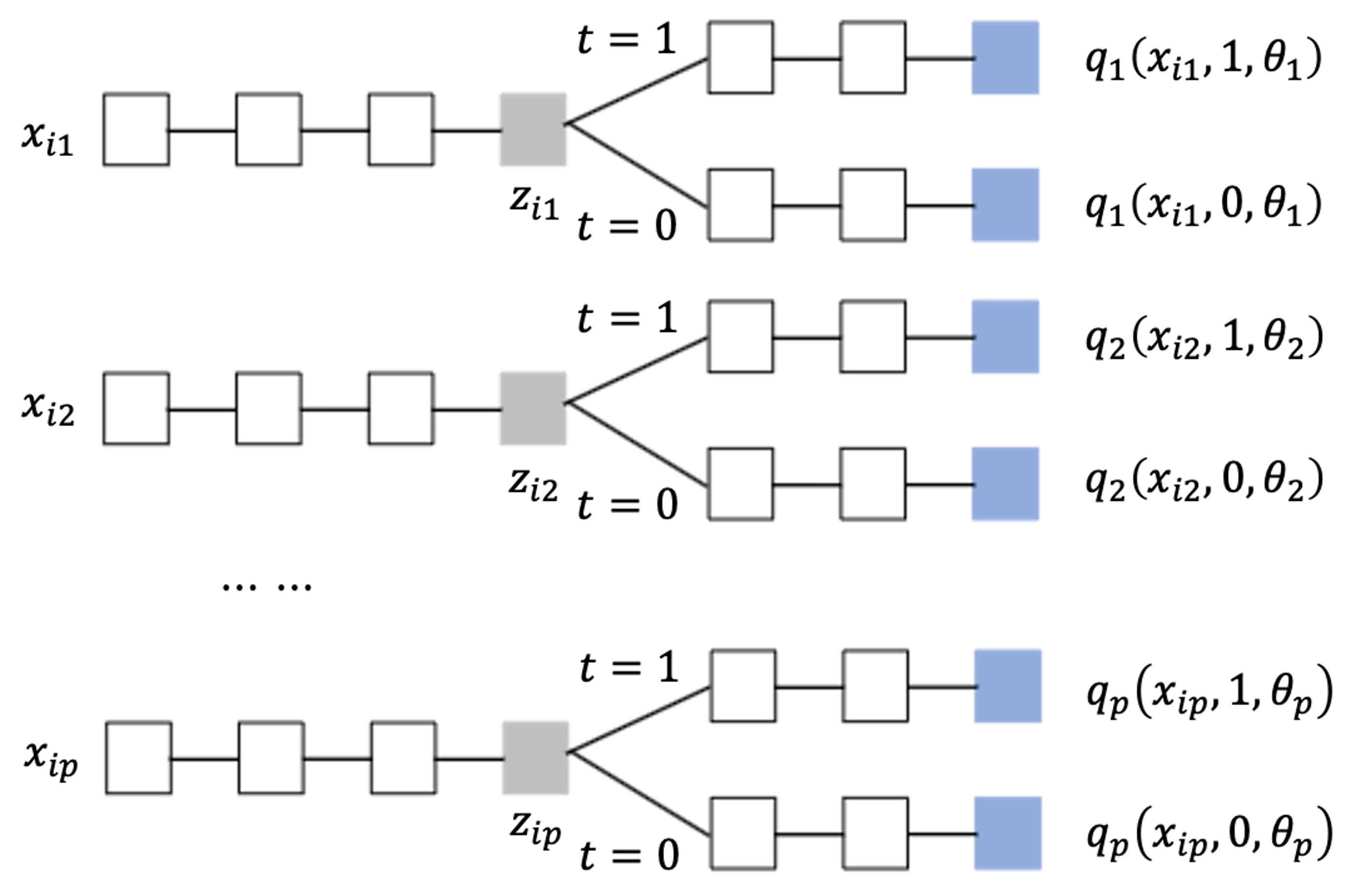}
  \end{center}
  \caption{Architecture of additive neural network for predicting factual and counterfactual outcomes.}
  \label{fig: architecture}
\end{wrapfigure}

While deep learning models offer very impressive predictive performance, their power often comes at the expense of interpretability, or lack thereof. To balance the trade-off between predictive performance and interpretability, \citet{agarwal2020neural} proposed the neural additive models (NAMs) combining the expressivity of DNNs with the interpretability of generalized additive models. Such models can be particularly useful for high stakes decision-making domains such as medicine where interpretability is highly desired or even necessary.  As such, we propose to use NAMs to predict factual and counterfactual outcomes as follows,
\begin{align*}
    \hat{Y}_i(a) = Q^{\text{NAM}}(\mathbf{X}_i, a; \theta) = \sum_{j=1}^p q_j (x_{ij},a;\theta_j)
\end{align*}
where $X_{i} = (x_{i1},\cdots,x_{ip}) \in \mathbb{R}^p$, $Q^{\text{NAM}}(\cdot,\cdot;\theta)$ is the fitted model, and $q_j$ is a univariate shape function.  Figure \ref{fig: architecture} shows the architecture of our additive neural network models, where $z_{ij} \in \mathbb{R}$ is the shared representation for $q_j(\cdot,1;\theta_j)$ and $q_j(\cdot,0;\theta_j)$. We can replace the $Q^{\text{NN}}$ with $Q^{\text{NAM}}$ in the objective function (\ref{equ: obj}) as appropriate.

\subsection{Covariate Balancing Score Weighted Regularization}

To further improve performance, we propose a refinement of the objective function in Section~\ref{sec:objective}  through weighted regularization, which is motivated by the following weighted estimator of ATE,
\begin{equation}\label{equ:weighted estimator}
    \hat{\tau}_{\mathbf{w}} = \frac{1}{n_1} \sum_{i=1}^n w_i Y_i A_i - \frac{1}{n_0} \sum_{i=1}^n w_i Y_i (1 - A_i).
\end{equation}

We define the regularization term $\gamma (\mathbf{Y}_i, A_i, \mathbf{X}_i; \theta, \epsilon)$ as follows,
\[\gamma (\mathbf{Y}_i, A_i, \mathbf{X}_i; \theta, \epsilon) = (Y_i - \tilde{Q}(\mathbf{X}_i, A_i, \theta))^2.\]
Here, $\tilde{Q}(\mathbf{X}_i, A_i; \theta) = 
    Q^{\text{NN}}(\mathbf{X}_i, A_i; \theta) + \epsilon \left( \frac{w_i A_i n}{n_1} - \frac{w_i (1-A_i)n}{n_0}\right)$, where the second term can be viewed as a calibration term and $\epsilon$ is a tuning parameter.
    
Then the refined objective function is defined as follows,
\begin{equation}
\begin{split}
    \hat{\theta}, \hat{\mathbf{w}}, \hat{\epsilon} &= \argmin_{\theta,\mathbf{w},\epsilon} \hat{R}(\mathbf{X},\mathbf{A};\theta, \mathbf{w}) 
+ \beta \frac{1}{n} \sum_{i=1}^n \gamma (\mathbf{Y}_i, A_i, \mathbf{X}_i; \theta, \epsilon)    \\
    s.t. \quad &\sum_{i=1}^n w_i A_i = n_1, \quad \sum_{i=1}^n w_i (1 - A_i) = n_0, w_i \geq 0 \\
\end{split}  \label{equ:modified_objective}  
\end{equation}
where $\hat{R}(\mathbf{X},\mathbf{A};\theta, \mathbf{w}) $ is defined in (\ref{equ: obj}), and $\beta \in \mathbb{R}^+$ is a hyperparameter. It follows that our downstream estimator for ATE becomes
\begin{align*}
        \hat{\tau}^{\text{wreg}} = \frac{1}{n} \sum_{i=1}^n \left( \tilde{Q}(\mathbf{X}_i, 1; \theta) - \tilde{Q}(\mathbf{X}_i, 0; \theta) \right).
\end{align*}

The idea behind the weighted regularization process can be understood as follows. Recall the general procedure for estimating treatment effect has two steps: 1) fit models for the conditional outcome and balancing score; 2) plug the fitted models into a downstream estimator.  Of note, there is a large body of literature on semi-parametric statistical methods for estimation of ATE \citep{kennedy2016semiparametric}. In this work, we focus on the non-parametric estimator in (\ref{equ:weighted estimator}) that uses energy distance balancing scores. Since this estimator does not require correct specification of the treatment assignment model, it is more robust than the Dragonnet model and target regularization in \citet{shi2019adapting}.

Of particular note, under some mild conditions, the estimator $\hat{\tau}$ is expected to have some desirable asymptotic properties such as
\begin{itemize}
    \item Robustness of the estimator, since the energy distance balancing score weight is non-parametric;
    \item Efficiency: $\hat{\tau}$ has the lowest variance of any consistent estimator of $\tau$. This means that $\hat{\tau}$ is the most data efficient estimator.   
\end{itemize}
In particular, two important conditions are needed for these properties to hold: 1) $\hat{Q}$ and $\hat{\mathbf{w}}$ are consistent estimators of outcomes and balancing score; 2) the following non-parametric estimating equation is satisfied:
\begin{equation} \label{equ: estimating equ}
    \begin{split}
       & \frac{1}{n}  \sum_{i=1}^n \Bigg(  \hat{Q}(\mathbf{X}_i, 1; \theta) - \hat{Q}(\mathbf{X}_i, 0; \theta) + \left( \frac{w_i A_i n}{n_1} - \frac{w_i (1-A_i)n}{n_0}      \right)(Y_i - \hat{Q}(\mathbf{X}_i, A_i; \theta))  - \tau \Bigg)  = 0
    \end{split}
\end{equation}
More details can be found in \citet{van2011targeted} and \citet{chernozhukov2017double}. The most important observation here is that minimizing the modified objective~\eqref{equ:modified_objective} would force $\tilde{Q}$, $\hat{\mathbf{w}}$, and $\hat{\tau}^{\text{wreg}} $ to satisfy the estimating equation (\ref{equ: estimating equ}) because
\begin{align*}
        \frac{\partial}{\partial \epsilon} \left( \hat{R}(\mathbf{X};\theta, \mathbf{w}) 
 + \beta \frac{1}{n} \sum_{i=1}^n \gamma (\mathbf{Y}_i, A_i, \mathbf{X}_i; \theta, \epsilon) \right) = 0.
\end{align*}
As long as $\tilde{Q}$ and $\hat{\mathbf{w}}$ are consistent, by imposing the weighted regularization process, the estimator $\hat{\tau}^{\text{wreg}}$ is expected to have the aforementioned desirable asymptotic properties. More detailed discussions about the rationale for this regularization process can be found in \citet{shi2019adapting}.

\section{Numerical Experiments} \label{sec: simulation}

We conduct numerical experiments to demonstrate the advantages of our proposed model compared to current state-of-the-art. Since ground truth causal effects are in general unavailable in real-world data, our experiments are conducted using semi-synthetic data derived from the Infant Health and Development Program (IHDP) and from the 2018 Atlantic Causal Inference Conference (ACIC) competition, two popular benchmark datasets for assessing causal inference methods.

\underline{\textbf{IHDP.}}  IHDP is a semi-synthetic dataset constructed from the Infant Health and Development Program, a randomized experiment started in 1985. \citet{hill2011bayesian} introduced this data in their 2011 paper. In this program, the treated group was provided with intensive high-quality child care and home visits from trained specialists. At the end of the intervention, when the children were three years old, the intervention was deemed successful in raising cognitive test scores of treated children in comparison to controls. The dataset includes numerous measurements on participants such as birth weight, neonatal health index, first born weeks born preterm, head circumference, sex and twin status. It also contains information of mother's behaviors engaged in during the pregnancy such as smoking status, alcohol usage and drug usage. The data generating process of the synthetic part of the data can be found in \citet{hill2011bayesian}. The data has 747 observations with 26 features.

%\subsection{ACIC 2018}

\underline{\textbf{ACIC 2018.}} This dataset was developed from the 2018 Atlantic Causal Inference Conference competition. The dataset include real-world clinical measurements taken from the Linked Birth and Infant Death Data (LBIDD) \citep{mathews1998infant}, e.g., infant mortality statistics from the linked birth/infant death data set (linked file) for the 1999 period, broken down by a variety of maternal and infant characteristics. In addition, it includes demographics of mothers and infants such as education, prenatal care, race, birth weight and days of born preterm. The outcome is the mortality rate. The treatment/exposure is the race of mothers (black or white). The synthetic part of the data is generated from 63 distinct data generating process settings. We follow the same selection criterion as in \citet{shi2019adapting}, that is, we randomly pick 3 datasets of size either 5,000 and 10,000.

\subsection{Interpretability} \label{subsec: interpret}

The interpreability of NAMs is due, in part, to the fact that it enables visualization of the impact of individual features on estimation of ATE. Since each feature is handled independently by a learned shape function parameterized by a neural net, plotting the individual shape functions offers a natural way to visualize their impact. These shape function plots can provide an exact description of how NAMs estimate the ATE, not merely just an explanation. This helps, for example, an decision-maker, say from medicine, to understand how to interpret the models and understand exactly how the estimated ATE is associated with risk factors. 

\begin{wrapfigure}{r}{0.5\textwidth}
  \begin{center}
    \includegraphics[width=0.48\textwidth]{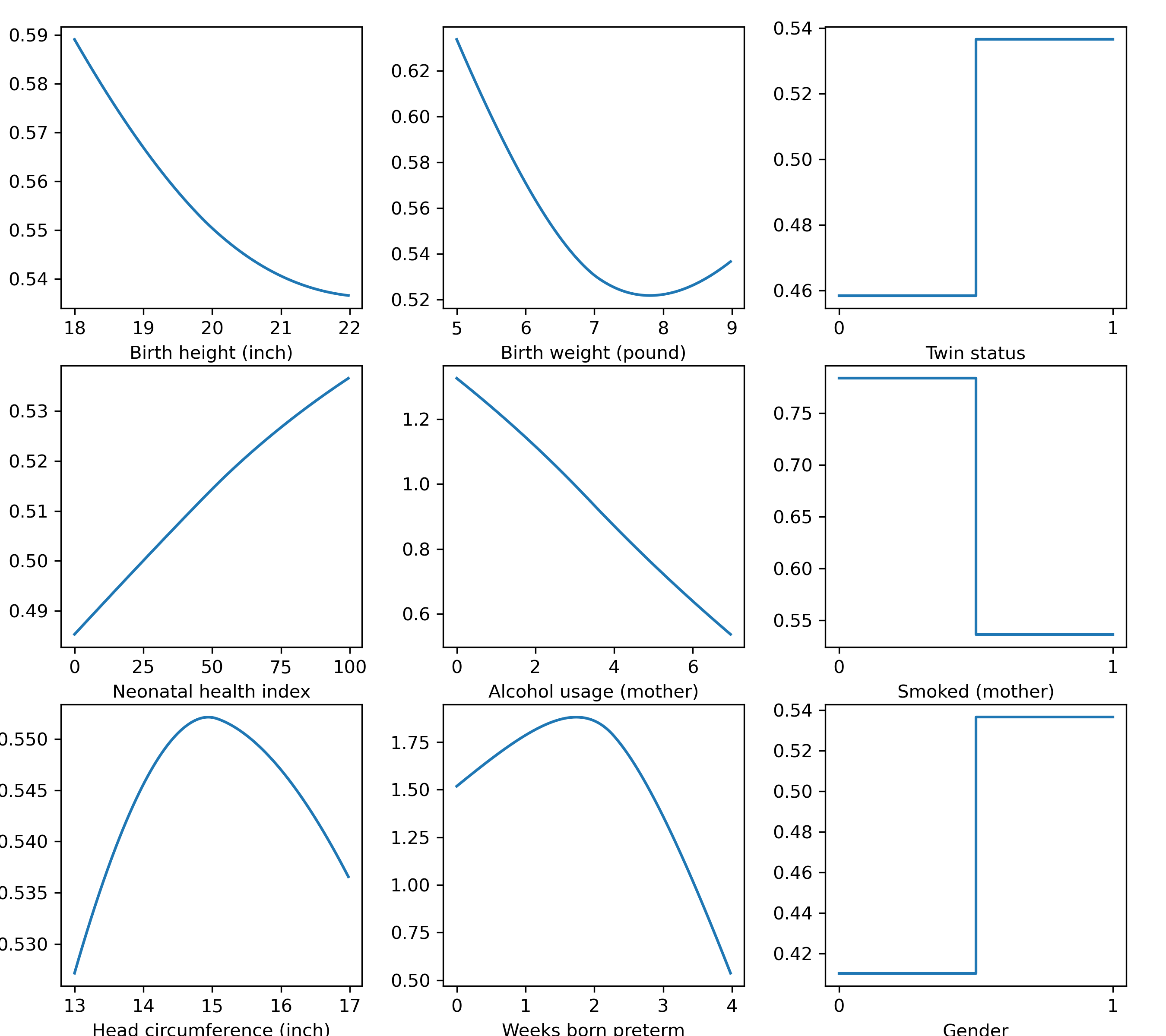}
  \end{center}
  \caption{Impact of selected features on estimation of treatment effect on cognitive test scores in IHDP using NAM models. The $y-$axis represents the average treatment effect that attributed to each selected feature. (Gender: 0-male, 1-female) }
 \label{fig:interpret}
\end{wrapfigure}

We analyze the IHDP data using our proposed model. Figure \ref{fig:interpret} presents the aforementioned plots  to demonstrate how to interpret the results from NAMs. First, by removing the mean score for each graph (i.e., each feature) across the entire training dataset, we set the average score for each plot in Figure \ref{fig:interpret} (i.e., each feature) to zero. A single bias term is then added to the model to make individual shape functions identifiable, so that the average prediction across all data points matches the observed baseline. This makes the interpretation of each term much easier.

In Figure \ref{fig:interpret}, we plot each learned shape function $q_j (x_{ij}, 1; \theta_j) - q_j (x_{ij}, 0; \theta_j)$ against $x_{ij}$ for an ensemble of neural additive model using blue line. This enables us to determine when the ensemble models learned the same shape function and when they diverged. By plotting each shape function, we are able to interpret how model estimates treatment effect based on the contributions from each observed covariate such as birth weight, birth height and neonatal health index.

\subsection{Treatment Effect Estimation} \label{subsec: models}

We further compare our proposed approach (with/without weighted regularization process) with several current state-of-the-art methods for estimating ATE: 1) BART, the Bayesian Additive Regression Tree (BART) by \citet{chipman2010bart}; 2) CBPS, covariate balancing propensity score (CBPS) by \citet{imai2014covariate};  3) CF, the causal forest model by \citet{dowhy} motivated by generalized random forest from \citet{athey2019generalized}; 4) SITE, a local similarity
preserved ITE estimation by \citet{yao2018representation};  5) GANITE, generative adversarial nets for the estimation of ITE by \citet{yoon2018ganite};  6) Perfect Matching, a method for training neural networks for counterfactual inference by \citet{schwab2018perfect};  7)  DRAGONNET, the deep neural network structure model by \citet{shi2019adapting} with target regularization;  8) TARNET, the deep neural network model by \citet{shalit2017estimating};  9) FlexTENet, flexible approach for CATE estimation by \citet{curth2021inductive}.

%\subsection{IHDP} \label{subsec: IHDP}

%\subsection{Results}

\underline{\textbf{Results}} Our experiment results in analyses of the IHDP and ACIC datasets are summarized in Table~\ref{table:1}. For IHDP, we randomly split the data into training set, testing set and validating set with the equal proportion (i.e. $1/3$ for each). For ACIC, we use all the data as training set and estimate the ATE. We report the mean absolute error $|\hat{\tau} - \tau|$ in Table~\ref{table:1}. Our results show that our proposed model with weighted regularization outperforms all the existing models  included as it has the lowest bias of estimation for ATE in both the IHDP and ACIC 2018 datasets. Actually,  even our base model~\eqref{equ: obj} without weighted regularization yields nearly close to the best performance among all the other methods, suggesting the advantage of our new objective function in Section~\ref{sec:objective}.  In addition, our results also shows that the weighted regularization in the refined model~\eqref{equ:modified_objective} improves the performance of our base model~\eqref{equ: obj}, and the improvement is more pronounced for ACIC. Moreover, in comparison to fully-connected neural network models, NAMs are inherently interpretable while suffering a little loss in estimation accuracy when applied to both IHDP and ACIC, demonstrating the trade-off between interpreability and accuracy. Our model with regularization and NAMs outperforms all the other models with NAMs, which is particularly pronounced in ACIC.

\begin{table}[!tb]
\caption{Performance on estimation of ATE in terms of the mean absolute error $|\hat{\tau} - \tau|$ across 200 datasets.  For all deep neural network models, there are 3 hidden layers with 100 hidden units; the hyperparameter $\alpha,\beta$ are 0.05 and 1. The deep neural network models are trained using stochastic gradient descent with momentum. The causal forest is implemented using the package from \citet{dowhy}; BART is implemented using the package from \citet{package}. For comparisons, the deep neural network models used in the methods denoted by $^*$ are changed to NAMs.  }
\label{table:1}
\vskip 0.15in
\begin{center}
\begin{small}
\begin{sc}
\begin{tabular}{lccccr}
\toprule
Method & IHDP & ACIC   \\
\midrule
\textbf{CLASSICAL ML MODELS } \\
\quad BART \citep{chipman2010bart}         & $0.95 \pm .04$  & $1.32 \pm .04 $\\
\quad CBPS \citep{imai2014covariate}                 & $0.84 \pm .03$ & $1.25 \pm .03$    \\
\quad CF \citep{dowhy} & $0.79 \pm .03$ & $1.15 \pm .03$ \\

\textbf{DEEP LEARNING MODELS } \\
\quad SITE \citep{yao2018representation}                                    & $0.32 \pm .02$ & $0.51 \pm .02$ \\
\quad GANITE \citep{yoon2018ganite}                                           & $0.51 \pm .03$ & $0.81 \pm .03$ \\
\quad Perfect Match \citep{schwab2018perfect}                                           & $0.40 \pm .02$ & $0.71 \pm .02$ \\
\quad DRAGONNET \citep{shi2019adapting}               & $0.20 \pm .01$ & $0.35 \pm .01$\\
\quad DRAGONNET$^*$ \citep{shi2019adapting}               & $0.31 \pm .01$ & $0.56 \pm .01$\\
\quad TARNET$^*$ \citep{shalit2017estimating}                    & $0.46 \pm .02$ & $0.77 \pm .02$\\
\quad FlexTENet$^*$ \citep{curth2021inductive}                                           & $0.49 \pm .03$ & $0.79 \pm .03$ \\
\quad Our model$^*$ (without regularization)                    & $\mathbf{0.39 \pm .02} $ & $\mathbf{0.53 \pm .02}$ \\
\quad Our model$^*$ (with regularization)                    & $\mathbf{0.29 \pm .02} $ & $\mathbf{0.46 \pm .02} $\\
\quad Our model (without regularization)                    & $\mathbf{0.27 \pm .02} $ & $\mathbf{0.44 \pm .02}$ \\
\quad Our model (with regularization)                    & $\mathbf{0.18 \pm .02} $ & $\mathbf{0.32 \pm .02} $\\
\bottomrule
\end{tabular}
\end{sc}
\end{small}
\end{center}
\vskip -0.1in
\end{table}

\section{Discussion and Conclusion} \label{sec: discussion and conclusion}

In this paper, we propose a new covariate-balancing-aware interpretable deep learning approach for estimating ATE. By leveraging appealing properties of the weighted energy distance, we derive a generalization error bound for the bias of estimation of ATE that is sharper than the existing work. Our theoretical analysis motivated the development of our new models including a new objective function and a energy distance propensity score weighted regularization. In addition,to balance trade-off between the predictive performance and interpreability, neural additive models can be used for predicting potential outcomes in our approach. Our numerical results including ablation experiments demonstrate the advantages of the proposed loss function and the weighted regularization. 

There are some promising future extensions of our work. For example, a possible extension is to incorporate instrumental variables for the explanation for hidden confounders. Recall our work is under ``strong ignorability'' assumption. \citet{hartford2017deep}  equipped their model with instrumental variables for counterfactual prediction under additive hidden bias model assumption. However, such assumption is somehow not applicable in most real cases in comparison to traditional instrumental variables assumptions such as mean independence assumption or monotonicity assumption. Conducting theoretical analysis for the bias of interval estimation of ATE under instrumental variable framework will be our next goal.   

%First, we may improve the expressivity and predictive performance of the neural additive model used for estimation of treatment effect by incorporating higher-order feature interactions. In the Bayesian causal forest by \citet{hahn2020bayesian}, they introduce an interaction term between propensity score and observed covariates to provide an adequate control over the strength of regularization over effect heterogeneity so that the bias of treatment effect is reduced. This idea could be applied to our setup by introducing the interaction term between balancing score and observed covariates.  

\bibliographystyle{apalike}
\bibliography{main.bib}

\newpage

%%% BEGIN INSTRUCTIONS %%%
%The checklist follows the references.  Please
%read the checklist guidelines carefully for information on how to answer these
%questions.  For each question, change the default \answerTODO{} to \answerYes{},
%\answerNo{}, or \answerNA{}.  You are strongly encouraged to include a {\bf
%justification to your answer}, either by referencing the appropriate section of
%your paper or providing a brief inline description.  For example:
%\begin{itemize}
  %\item Did you include the license to the code and datasets? \answerYes{See Section~\ref{gen_inst}.}
  %\item Did you include the license to the code and datasets? \answerNo{The code and the data are proprietary.}
  %\item Did you include the license to the code and datasets? \answerNA{}
%\end{itemize}
%Please do not modify the questions and only use the provided macros for your
%answers.  Note that the Checklist section does not count towards the page
%limit.  In your paper, please delete this instructions block and only keep the
%Checklist section heading above along with the questions/answers below.
%%% END INSTRUCTIONS %%%

\newpage

\section{Supplementary Materials for ``Covariate-Balancing-Aware Interpretable Deep Learning Models for Treatment Effect Estimation"} \label{supp}

\subsection{Covariate Balancing Propensity Score with Deep Neural Network}

Covariate balancing propensity score (CBPS) can be viewed as an adjustment for propensity score model based on the distributions of treated and control units. To equip deep neural network models with covariate balancing propensity score, let $Q^{\text{NN}}(\mathbf{X}_i,A_i;\theta)$ denote the predicted outcomes using neural network, and $g(\mathbf{X}_i;\theta)$ be the estimated propensity score, then we train our neural network by minimizing the objective function

\begin{align*}
            & \hat{\theta}  = \argmin_{\theta}\hat{R}(\mathbf{X}, \mathbf{A};\theta) ,\quad \text{where}\\
 \hat{R}(\mathbf{X}, \mathbf{A};\theta)   &= \frac{1}{n}\sum_{i=1}^n(Q^{\text{NN}}(\mathbf{X}_i,A_i;\theta)-Y_{i})^{2} + \bar h_\theta (\mathbf{A},\mathbf{X})^\top \Sigma_\theta (\mathbf{A}, \mathbf{X})^{-1} \bar h_\theta (\mathbf{A},\mathbf{X}) \\
 \bar h_\theta (\mathbf{A},\mathbf{X}) &= \frac{1}{n} \sum_{i=1}^n h_\theta (A_i,\mathbf{X}_i) \\
 h_\theta (A_i,\mathbf{X}_i) &= \binom{\frac{\partial g }{\partial \theta} ( \mathbf{X}_i; \theta) }{\frac{A_i - g(\mathbf{X}_i;\theta)}{g(\mathbf{X}_i;\theta) (1 - g(\mathbf{X}_i;\theta))} \mathbf{X}_i  } \\
 \Sigma_\theta (\mathbf{A}, \mathbf{X}) &= \frac{1}{n} \sum_{i=1}^N \mathbb{E}\left( g(\mathbf{X}_i;\theta) g(\mathbf{X}_i;\theta)^\top | \mathbf{X}_i        \right).
\end{align*}

It is motivated by the \emph{Method of Moments}. We match the first moment of the distributions of treated units and control units. Details of CBPS can be found in \citet{imai2014covariate}.

\subsection{Proof of Lemma 1}
\begin{proof}
According to the concept of expected factual and counterfactual losses in Definition \ref{def: f and cf loss}, the concepts of expected factual (counterfactual) treated and control losses are immediately followed: 
\begin{align*}
    &R_F (\hat{\mu}_1) = \int_{\mathbb{R}^p} l_{\hat{\mu}_1} (\mathbf{x}) \, d F_1 (\mathbf{x})\\
    &R_F (\hat{\mu}_0) = \int_{\mathbb{R}^p} l_{\hat{\mu}_0} (\mathbf{x}) \, d F_0 (\mathbf{x})\\
    &R_{CF} (\hat{\mu}_1) = \int_{\mathbb{R}^p} l_{\hat{\mu}_1} (\mathbf{x}) \, d F_0 (\mathbf{x})\\
    &R_{CF} (\hat{\mu}_0) = \int_{\mathbb{R}^p} l_{\hat{\mu}_0} (\mathbf{x}) \, d F_1 (\mathbf{x})
\end{align*}
The connections between expected factual (counterfactual) treated losses and expected factual (counterfactual)  control losses are: 
\begin{align*}
    &R_{F} (\hat{\mu}_a) = P_0 \, R_{F} (\hat{\mu}_0) + P_1 \, R_{F} (\hat{\mu}_1)\\
    &R_{CF} (\hat{\mu}_a) = P_1 \, R_{CF} (\hat{\mu}_0) + P_0 \, R_{CF} (\hat{\mu}_1)
\end{align*}
where $P_1 \equiv \mathbb{P}(A = 1), P_0 \equiv \mathbb{P}(A = 0)$. 

Then if we do subtraction between $R_{CF} (\hat{\mu}_a)$ and $R_{F} (\hat{\mu}_a) $, then
\begin{align*}
        \quad & R_{CF} (\hat{\mu}_a) - R_{F} (\hat{\mu}_a) = \left[P_0 \, R_{CF} (\hat{\mu}_1) + P_1 \, R_{CF} (\hat{\mu}_0)\right]- \left[P_0 \, R_F (\hat{\mu}_1)  + P_1 \, R_{F} (\hat{\mu}_0)\right]\\
        =& P_0 \left[R_{CF} (\hat{\mu}_1) - R_F (\hat{\mu}_1)\right] + P_1 \left[R_{CF} (\hat{\mu}_0) - P_1 R_{F} (\hat{\mu}_0)\right]\\
        =& P_0 \left[ \int_{\mathbb{R}^p} l_{\hat{\mu}_1} (\mathbf{x}) \, d F_0 (\mathbf{x}) - \int_{\mathbb{R}^p} l_{\hat{\mu}_1} (\mathbf{x}) \, d F_1 (\mathbf{x}) \right] + P_1 \left[ \int_{\mathbb{R}^p} l_{\hat{\mu}_0} (\mathbf{x}) \, d F_1 (\mathbf{x}) - \int_{\mathbb{R}^p} l_{\hat{\mu}_0} (\mathbf{x}) \, d F_0 (\mathbf{x}) \right]\\
        =& -P_0 \int_{\mathbb{R}^p} l_{\hat{\mu}_1} (\mathbf{x}) \, d[F_1 - F_0] (\mathbf{x}) + P_1 \int_{\mathbb{R}^p} l_{\hat{\mu}_0} (\mathbf{x}) \, d[F_1 - F_0] (\mathbf{x})\\
        =& \int_{\mathbb{R}^p} [-P_0 l_{\hat{\mu}_1} + P_1 l_{\hat{\mu}_0}] (\mathbf{x}) \, d[F_1 - F_0] (\mathbf{x})\\
        =& \int_{\mathbb{R}^p} [-P_0 l_{\hat{\mu}_1} + P_1 l_{\hat{\mu}_0}] (\mathbf{x}) \, d[F - F_0] (\mathbf{x})- \int_{\mathbb{R}^p} [-P_0 l_{\hat{\mu}_1} + P_1 l_{\hat{\mu}_0}] (\mathbf{x}) \, d[F - F_1] (\mathbf{x}).
\end{align*}
Let $g(\mathbf{x}) = [-P_0 l_{\hat{\mu}_1} + P_1 l_{\hat{\mu}_0}] (\mathbf{x})$, accoring to Theorem 4 (Koksma-Hlawka) by \citet{mak2018support}:
\begin{align*}
            \quad & \left|\int_{\mathbb{R}^p} [-P_0 l_{\hat{\mu}_1} + P_1 l_{\hat{\mu}_0}] (\mathbf{x}) \, d[F - F_0] (\mathbf{x})\right| = \left|\int_{\mathbb{R}^p} g (\mathbf{x}) \, d[F - F_0] (\mathbf{x})\right|\\
        \leq & \|\partial^p g / \partial \mathbf{x}^p  \|_2 \left(\int_{\mathbb{R}^p}(F(\mathbf{x}) - F_0(\mathbf{x}))^2 d\mathbf{x}\right)^{1/2} =  \frac{\|\partial^p g / \partial \mathbf{x}^p  \|_2}{\sqrt{2}} \sqrt{\mathcal{E}(F, F_0)} 
        =  C_{l_{\hat{\mu}_a} } \sqrt{\mathcal{E}(F, F_0)}
\end{align*}
Similarly, 
\begin{align*}
            \quad & \left|\int_{\mathbb{R}^p} [-P_0 l_{\hat{\mu}_1} + P_1 l_{\hat{\mu}_0}] (\mathbf{x}) \, d[F - F_1] (\mathbf{x})\right|
        \leq \frac{\|\partial^p g / \partial \mathbf{x}^p  \|_2}{\sqrt{2}} \sqrt{\mathcal{E}(F, F_1)} 
         =  C_{l_{\hat{\mu}_a} } \sqrt{\mathcal{E}(F, F_1)}
\end{align*}
where $C_{l_{\hat{\mu}_a} } = \|\partial^p [-P_0 l_{\hat{\mu}_1} + P_1 l_{\hat{\mu}_0} ](\mathbf{x}) /\partial \mathbf{x}^p  )  \|_2/\sqrt{2}$ is a constant.
According to the triangle inequality, 
\begin{align*}
            \quad & R_{CF} (\hat{\mu}_a) - \left[P_0 \, R_F (\hat{\mu}_1)  + P_1 \, R_{F} (\hat{\mu}_0)\right]
        \leq \left| R_{CF} (\hat{\mu}_a) - \left[P_0 \, R_F (\hat{\mu}_1)  + P_1 \, R_{F} (\hat{\mu}_0)\right] \right|\\
        =& \Big| \int_{\mathbb{R}^p} [-P_0 l_{\hat{\mu}_1} + P_1 l_{\hat{\mu}_0}] (\mathbf{x}) \, d[F - F_0] (\mathbf{x}) - \int_{\mathbb{R}^p} [-P_0 l_{\hat{\mu}_1} + P_1 l_{\hat{\mu}_0}] (\mathbf{x}) \, d[F - F_1] (\mathbf{x}) \Big|\\
        \leq& \left|\int_{\mathbb{R}^p} [-P_0 l_{\hat{\mu}_1} + P_1 l_{\hat{\mu}_0}] (\mathbf{x}) \, d[F - F_0] (\mathbf{x})\right| + \left|\int_{\mathbb{R}^p} [-P_0 l_{\hat{\mu}_1} + P_1 l_{\hat{\mu}_0}] (\mathbf{x}) \, d[F - F_1] (\mathbf{x})\right|\\
        \leq& C_{l_{\hat{\mu}_a} } \left( \sqrt{\mathcal{E}(F, F_0)} +  \sqrt{\mathcal{E}(F, F_1)} \right)
\end{align*}

To prove the result when $n \to \infty$ with weights, similarly, the weighted expected factual (counterfactual) treated and control losses and their relations are given as: 
\begin{align*}
            &R_F (\hat{\mu}_1, \mathbf{w}) = \int_{\mathbb{R}^p} l_{\hat{\mu}_1} (\mathbf{x}) \, d F_{n, 1, \mathbf{w}} (\mathbf{x}) \\
        &R_F (\hat{\mu}_0, \mathbf{w}) = \int_{\mathbb{R}^p} l_{\hat{\mu}_0} (\mathbf{x}) \, d F_{n, 0, \mathbf{w}} (\mathbf{x}) \\
        &R_{CF} (\hat{\mu}_1, \mathbf{w}) = \int_{\mathbb{R}^p} l_{\hat{\mu}_1} (\mathbf{x}) d F_{n, 0, \mathbf{w}} (\mathbf{x})\\
        &R_{CF} (\hat{\mu}_0, \mathbf{w}) = \int_{\mathbb{R}^p} l_{\hat{\mu}_0} (\mathbf{x}) d F_{n, 1, \mathbf{w}} (\mathbf{x})\\
        &R_F (\hat{\mu}_a, \mathbf{w}) = P_0 \, R_F (\hat{\mu}_0, \mathbf{w}) + P_1 \, R_F (\hat{\mu}_1, \mathbf{w})\\
        &R_{CF} (\hat{\mu}_a, \mathbf{w}) = P_1 \, R_{CF} (\hat{\mu}_0, \mathbf{w}) + P_0 \, R_{CF} (\hat{\mu}_1, \mathbf{w}).
\end{align*}
Then again, we can do subtraction between $R_{CF} (\hat{\mu}_a, \mathbf{w})$ and $R_{F} (\hat{\mu}_a, \mathbf{w})$: 
\begin{align*}
            \quad & R_{CF} (\hat{\mu}_a, \mathbf{w}) - R_{F} (\hat{\mu}_a, \mathbf{w})\\
        =& \int_{\mathbb{R}^p} [-P_0 l_{\hat{\mu}_1} + P_1 l_{\hat{\mu}_0}] (\mathbf{x}) \, d[F_{n} - F_{n, 0, \mathbf{w}}] (\mathbf{x}) - \int_{\mathbb{R}^p} [-P_0 l_{\hat{\mu}_1} + P_1 l_{\hat{\mu}_0}] (\mathbf{x}) \, d[F_{n} - F_{n, 1, \mathbf{w}}] (\mathbf{x}).
\end{align*}
 Then by Lemma 3.3 in \citet{huling2020energy}, 
\begin{align*}
            \quad &\left| \int_{\mathbb{R}^p} [-P_0 l_{\hat{\mu}_1} + P_1 l_{\hat{\mu}_0}] (\mathbf{x}) \, d[F_{n} - F_{n, 0, \mathbf{w}}] (\mathbf{x}) \right| = C_{l_{\hat{\mu}_a} } \sqrt{\mathcal{E}(F_{n}, F_{n, 0, \mathbf{w}})}
\end{align*}
and 
\begin{align*}
            \quad &\left| \int_{\mathbb{R}^p} [-P_0 l_{\hat{\mu}_1} + P_1 l_{\hat{\mu}_0}] (\mathbf{x}) \, d[F_{n} - F_{n, 1, \mathbf{w}}] (\mathbf{x}) \right| = C_{l_{\hat{\mu}_a} } \sqrt{\mathcal{E}(F_{n}, F_{n, 1, \mathbf{w}})}
\end{align*}
Finally, 
\begin{align*}
            \quad & R_{CF} (\hat{\mu}_a, \mathbf{w}) - R_{F} (\hat{\mu}_a, \mathbf{w})
        \leq \left| R_{CF} (\hat{\mu}_a, \mathbf{w}) - R_{CF} (\hat{\mu}_a, \mathbf{w}) - R_{F} (\hat{\mu}_a, \mathbf{w}) \right|\\
        =& \Big|\int_{\mathbb{R}^p} [-P_0 l_{\hat{\mu}_1} + P_1 l_{\hat{\mu}_0}] (\mathbf{x}) \, d[F_{n} - F_{n, 0, \mathbf{w}}] (\mathbf{x}) - \int_{\mathbb{R}^p} [-P_0 l_{\hat{\mu}_1} + P_1 l_{\hat{\mu}_0}] (\mathbf{x}) \, d[F_{n} - F_{n, 1, \mathbf{w}}] (\mathbf{x})\Big| \\
        \leq& \left| \int_{\mathbb{R}^p} [-P_0 l_{\hat{\mu}_1} + P_1 l_{\hat{\mu}_0}] (\mathbf{x}) \, d[F_{n} - F_{n, 0, \mathbf{w}}] (\mathbf{x}) \right| - \left| \int_{\mathbb{R}^p} [-P_0 l_{\hat{\mu}_1} + P_1 l_{\hat{\mu}_0}] (\mathbf{x}) \, d[F_{n} - F_{n, 1, \mathbf{w}}] (\mathbf{x}) \right|\\
        \leq& C_{l_{\hat{\mu}_a} } \left( \sqrt{\mathcal{E}(F_{n}, F_{n, 0, \mathbf{w}})} +  \sqrt{\mathcal{E}(F_{n}, F_{n, 1, \mathbf{w}})} \right)
\end{align*}
\end{proof}

\subsection{Proof of Theorem 1}

Before proving Theorem~\ref{thm: bound error ATE}, it will be helpful to state the following definition and claims. 

\begin{definition}
The expected variance of $Y(a)$ with respect to the distribution $p(\mathbf{x},a)$ is
\begin{align*}
     \sigma_{Y(a)}^2 (p(\mathbf{x},a)) = \int_{\mathbb{R}^p \times \mathcal{Y}} (Y(a) - \mu_a(\mathbf{x}))^2 p(Y(a)|\mathbf{x})  dY(a) dF_a (\mathbf{x})
\end{align*}
and we further define
\begin{align*}
    &\sigma_{Y(a)}^2 = \min \{\sigma_{Y(a)}^2 (p(\mathbf{x},a)),\sigma_{Y(a)}^2 (p(\mathbf{x},1-a))  \} \\
    &\sigma_{Y}^2 = \min \{\sigma_{Y(1)}^2,\sigma_{Y(0)}^2  \}.
\end{align*}
\end{definition}

\begin{claim} \label{claim:1}
Given all conditions of Lemma \ref{lemma: CF}, the following two equalities hold:
\begin{align*}
            \int_{\mathbb{R}^p \times \{0, 1\}} \lVert \hat{\mu}_a(\mathbf{x}) - \mu_a(\mathbf{x}) \rVert_2^2 \, dF_a(\mathbf{x}) \, da = R_{F}(\hat{\mu}_a) - \sigma_{Y(a)}^2 (p(\mathbf{x},a))
\end{align*}
and
\begin{align*}
            \int_{\mathbb{R}^p \times \{0, 1\}} \lVert \hat{\mu}_a(\mathbf{x}) - \mu_a(\mathbf{x}) \rVert_2^2 \, dF_{1-a}(\mathbf{x}) \, da = R_{CF}(\hat{\mu}_a) - \sigma_{Y(a)}^2 (p(\mathbf{x}, 1-a))
\end{align*}
\end{claim}

\begin{proof}[Proof of Claim 1]
\begin{align*}
             R_{F}(\hat{\mu}_a) =& \int_{\mathbb{R}^p \times \{0, 1\} \times \mathcal{Y}} \lVert \hat{\mu}_a(\mathbf{x}) - Y(a) \rVert_2^2 \, p(Y(a)|\mathbf{x}) \, dF_{1-a}(\mathbf{x}) \, dY(a) \, da\\
        =& \int_{\mathbb{R}^p \times \{0, 1\} \times \mathcal{Y}} \lVert \hat{\mu}_a(\mathbf{x}) - \mu_a(\mathbf{x}) \rVert_2^2 \, p(Y(a)|\mathbf{x}) \, dF_{1-a}(\mathbf{x}) \, dY(a) \, da\\
        &+ \int_{\mathbb{R}^p \times \{0, 1\} \times \mathcal{Y}} \lVert \mu_a(\mathbf{x}) - Y(a) \rVert_2^2 \, p(Y(a)|\mathbf{x}) \, dF_{1-a}(\mathbf{x}) \, dY(a) \, da\\
        =& \int_{\mathbb{R}^p \times \{0, 1\} } \lVert \hat{\mu}_a(\mathbf{x}) - \mu_a(\mathbf{x}) \rVert_2^2 \, dF_{1-a}(\mathbf{x}) \, da + \sigma_{Y(1)}^2 (p(\mathbf{x},a)) + \sigma_{Y(0)}^2 (p(\mathbf{x},a))\\
        =& \int_{\mathbb{R}^p \times \{0, 1\} } \lVert \hat{\mu}_a(\mathbf{x}) - \mu_a(\mathbf{x}) \rVert_2^2 \, p(Y(a)|\mathbf{x}) \, dF_{1-a}(\mathbf{x}) \, da + \sigma_{Y(a)}^2 (p(\mathbf{x},a))
\end{align*}
Therefore, 
\begin{align*}
            \int_{\mathbb{R}^p \times \{0, 1\}} \lVert \hat{\mu}_a(\mathbf{x}) - \mu_a(\mathbf{x}) \rVert_2^2 \, dF_a(\mathbf{x}) \, da = R_{F}(\hat{\mu}_a) - \sigma_{Y(a)}^2 (p(\mathbf{x}, a))
\end{align*}
Similarly, 
\begin{align*}
            \int_{\mathbb{R}^p \times \{0, 1\}} \lVert \hat{\mu}_a(\mathbf{x}) - \mu_a(\mathbf{x}) \rVert_2^2 \, dF_{1-a}(\mathbf{x}) \, da = R_{CF}(\hat{\mu}_a) - \sigma_{Y(a)}^2 (p(\mathbf{x}, 1-a))
\end{align*}
\end{proof}

\begin{claim} \label{claim:2}
Given all conditions of Lemma \ref{lemma: CF}, the following inequality also holds:
\begin{align*}
    R_{ATE}(\hat{\mu}_a) \leq 2 \left( R_F(\hat{\mu}_a) + R_{CF}(\hat{\mu}_a) - 2 \sigma_Y^2 \right)
\end{align*}
\end{claim}

\begin{proof}[Proof of Claim 2]
\begin{align*}
            R_{ATE}(\hat{\mu}_a) =& \int_{\mathbb{R}^p} \lVert \hat{\tau}(\mathbf{x}) - \tau(\mathbf{x}) \rVert_2^2 \, dF(\mathbf{x})
        = \int_{\mathbb{R}^p} \lVert \left(\hat{\mu}_1(\mathbf{x}) - \hat{\mu}_0(\mathbf{x}) \right) - \left(\mu_1(\mathbf{x}) - \mu_0(\mathbf{x}) \right) \rVert_2^2 \, dF(\mathbf{x})\\
        =& \int_{\mathbb{R}^p} \lVert \left(\hat{\mu}_1(\mathbf{x}) - \mu_1(\mathbf{x})  \right) - \left(\hat{\mu}_0(\mathbf{x}) - \mu_0(\mathbf{x}) \right) \rVert_2^2 \, dF(\mathbf{x})\\
        \leq& 2\int_{\mathbb{R}^p} \lVert \hat{\mu}_1(\mathbf{x}) - \mu_1(\mathbf{x}) \rVert_2^2 \, dF(\mathbf{x}) + 2\int_{\mathbb{R}^p} \lVert \hat{\mu}_0(\mathbf{x}) - \mu_0(\mathbf{x}) \rVert_2^2 \, dF(\mathbf{x})\\
        =& 2\int_{\mathbb{R}^p} \lVert \hat{\mu}_1(\mathbf{x}) - \mu_1(\mathbf{x}) \rVert_2^2 \, dF_1(\mathbf{x}) + 2\int_{\mathbb{R}^p} \lVert \hat{\mu}_1(\mathbf{x}) - \mu_1(\mathbf{x}) \rVert_2^2 \, dF_0(\mathbf{x})\\
        &+ 2\int_{\mathbb{R}^p} \lVert \hat{\mu}_0(\mathbf{x}) - \mu_0(\mathbf{x}) \rVert_2^2 \, dF_1(\mathbf{x}) + 2\int_{\mathbb{R}^p} \lVert \hat{\mu}_0(\mathbf{x}) - \mu_0(\mathbf{x}) \rVert_2^2 \, dF_0(\mathbf{x})\\
        =& 2\int_{\mathbb{R}^p \times \{0, 1\}} \lVert \hat{\mu}_a(\mathbf{x}) - \mu_a(\mathbf{x}) \rVert_2^2 \, dF_a(\mathbf{x}) \, da + 2\int_{\mathbb{R}^p \times \{0, 1\}} \lVert \hat{\mu}_a(\mathbf{x}) - \mu_a(\mathbf{x}) \rVert_2^2 \, dF_{1-a}(\mathbf{x}) \, da\\
        \leq& 2\left(R_{F}(\hat{\mu}_a) - \sigma_{Y}^2\right) + 2\left(R_{CF}(\hat{\mu}_a) - \sigma_{Y}^2\right)
\end{align*}
The last line is proved by Claim \ref{claim:1}. 
\end{proof}

\begin{proof}[Proof of Theorem 1]
According to Lemma 1 and its proof, we have
\begin{align*}
            R_{CF} (\hat{\mu}_a) \leq P_0 \, R_F (\hat{\mu}_1)  + P_1 \, R_{F} (\hat{\mu}_0) + C_{l_{\hat{\mu}_a} } \left( \sqrt{\mathcal{E}(F, F_0)} +  \sqrt{\mathcal{E}(F, F_1)} \right)
\end{align*}
and 
\begin{align*}
    R_F(\hat{\mu}_a) = P_0 \, R_F(\hat{\mu}_0) + P_1 \, R_F(\hat{\mu}_1)
\end{align*}
then
\begin{align*}
            R_F(\hat{\mu}_a) + R_{CF}(\hat{\mu}_a) \leq& P_0 \, R_F (\hat{\mu}_1)  + P_1 \, R_{F} (\hat{\mu}_0) + C_{l_{\hat{\mu}_a} } \left( \sqrt{\mathcal{E}(F, F_0)} +  \sqrt{\mathcal{E}(F, F_1)} \right) \\
            &+ P_0 \, R_F(\hat{\mu}_0) + P_1 \, R_F(\hat{\mu}_1)\\
        =& R_F (\hat{\mu}_1)  + R_{F} (\hat{\mu}_0) + C_{l_{\hat{\mu}_a} } \left( \sqrt{\mathcal{E}(F, F_0)} +  \sqrt{\mathcal{E}(F, F_1)} \right)
\end{align*}
Therefore, 
\begin{align*}
            R_{ATE}(\hat{\mu}_a) \leq& 2\left(R_{CF}(\hat{\mu}_a) - \sigma_{Y}^2\right) + 2\left(R_{CF}(\hat{\mu}_a) - \sigma_{Y}^2\right)\\
        =& 2\big( R_F (\hat{\mu}_1)  + R_{F} (\hat{\mu}_0) + C_{l_{\hat{\mu}_a} } ( \sqrt{\mathcal{E}(F, F_0)} +  \sqrt{\mathcal{E}(F, F_1)}) - 2 \sigma_Y^2 \big) \\
        =& 2\left( R_F (\hat{\mu}_1)  + R_{F} (\hat{\mu}_0) + C_{l_{\hat{\mu}_a} } \Big(  \sqrt{\mathcal{E}(F, F_0)} +  \sqrt{\mathcal{E}(F, F_1)} \Big)  \right)
\end{align*}
To prove the expected loss in estimation of average treatment effect with weight, Claim \ref{claim: 3} will be given. The proof of Claim~\ref{claim: 3} is similar to Claim~\ref{claim:1} and Claim~\ref{claim:2}. 

\begin{claim} \label{claim: 3}
Under conditions of Lemma \ref{lemma: CF}, the following equalities hold
\begin{align*}
            \int_{\mathbb{R}^p \times \{0, 1\}} \lVert \hat{\mu}_a(\mathbf{x}) - \mu_a(\mathbf{x}) \rVert_2^2 \, d F_{n, a, \mathbf{w}} (\mathbf{x}) \, da =& R_F (\hat{\mu}_a, \mathbf{w}) - \sigma_{Y(a)}^2 (p(\mathbf{x}, a))\\
            \int_{\mathbb{R}^p \times \{0, 1\}} \lVert \hat{\mu}_a(\mathbf{x}) - \mu_a(\mathbf{x}) \rVert_2^2 \, d F_{n, 1-a, \mathbf{w}} (\mathbf{x}) \, da =& R_{CF} (\hat{\mu}_a, \mathbf{w}) - \sigma_{Y(a)}^2 (p(\mathbf{x}, 1-a)),
\end{align*}
and
\begin{align*}
            R_{ATE}(\hat{\mu}_a) =& R_{ATE, n}(\hat{\mu}_a) \leq 2 \left( R_F (\hat{\mu}_1,\mathbf{w})  + R_{F} (\hat{\mu}_0,\mathbf{w}) - 2 \sigma_Y^2 \right)
\end{align*}
\end{claim}

According to the result in Lemma \ref{lemma: CF} with weight, Claim \ref{claim:1}, Claim \ref{claim:2} and Claim \ref{claim: 3}, 
\begin{align*}
            \quad & R_F (\hat{\mu}_1,\mathbf{w})  + R_{CF} (\hat{\mu}_0,\mathbf{w})\\
            \leq & R_F (\hat{\mu}_1,\mathbf{w})  + R_{F} (\hat{\mu}_0,\mathbf{w}) + C_{l_{\hat{\mu}_a} } \sqrt{\mathcal{E}(F_{n}, F_{n, 0, \mathbf{w}})} + C_{l_{\hat{\mu}_a} } \sqrt{\mathcal{E}(F_{n}, F_{n, 1, \mathbf{w}})}
\end{align*}
Therefore, 
\begin{align*}
            R_{ATE}(\hat{\mu}_a) \leq & 2 \left(R_F (\hat{\mu}_1,\mathbf{w}) + R_{F} (\hat{\mu}_0,\mathbf{w}) - 2 \sigma_Y^2 \right)\\
        \leq & 2\left( R_F (\hat{\mu}_1,\mathbf{w}) + R_{F} (\hat{\mu}_0,\mathbf{w}) + C_{l_{\hat{\mu}_a} } \Big( \sqrt{\mathcal{E}(F_{n}, F_{n, 0, \mathbf{w}})} + \sqrt{\mathcal{E}(F_{n}, F_{n, 1, \mathbf{w}})} \Big) \right)  \\
\end{align*}
\end{proof}

\subsection{Experiment Details}

The experiments are run using R (version of 4.1.0) and Python (Version of 3.8.5) on Intel(R) 8 Cores(TM) i7-9700F CPU (3.00GHz). We use R package `\texttt{bartCause}' license GPL (>= 2) for BART, `\texttt{CBPS}' license GPL (>= 2) for CBPS, and `\texttt{grf}' license GPL-3 for Causal Forest. We also use Python codes for GANITE from \url{https://github.com/jsyoon0823/GANITE}, SITE from \url{https://github.com/Osier-Yi/SITE}, FlexTENet from \url{https://github.com/AliciaCurth/CATENets}, Perfect Match from \url{https://github.com/d909b/perfect_match}, and DRAGONNET from \url{https://github.com/claudiashi57/dragonnet}. IHDP data is available from \url{https://github.com/claudiashi57/dragonnet/tree/master/dat/ihdp/csv}, and ACIC 2018 data is avaialble from \url{https://www.synapse.org/#!Synapse:syn11294478/wiki/486304}.

\end{document}